\documentclass[runningheads]{llncs}
\usepackage[T1]{fontenc}
\usepackage{graphicx}
\usepackage{booktabs}
\usepackage[misc]{ifsym}
\usepackage{subcaption}
\usepackage{graphicx}
\usepackage{url}

\newcommand{\corr}{(\Letter)}
\newcommand{\ourmethod}{\textsc{GraphWeave}\xspace}
\usepackage{mwe}
\usepackage{bm, xspace, dsfont}
\usepackage{algorithm, algpseudocode}
\usepackage{xparse}
\usepackage{nicematrix}
\usepackage{amsmath}
\usepackage{amssymb}
\makeatletter
\NewDocumentCommand{\LeftComment}{s m}{%
  \Statex \IfBooleanF{#1}{\hspace*{\ALG@thistlm}}\(\triangleright\) #2}
\makeatother

\DeclareMathOperator*{\argmin}{arg\,min}

\usepackage{amsmath}
\newcommand{\bv}{\ensuremath{\bm v}\xspace}
\newcommand{\bw}{\ensuremath{\bm w}\xspace}
\newcommand{\bone}{\ensuremath{\bm 1}\xspace}

\begin{document}
\title{GraphWeave : Interpretable and Robust Graph Generation via Random Walk Trajectories}

\titlerunning{GraphWeave}

\author{Rahul Nandakumar \inst{1} \corr, Deepayan Chakrabarti \inst{1}}

\authorrunning{R. Nandakumar and D. Chakrabarti}

\institute{The University of Texas at Austin, Austin TX, 78713, USA
\email{\{rahul.nandakumar,deepay\}@utexas.edu}}

\maketitle              

\begin{abstract}
Given a set of graphs from some unknown family, we want to generate new graphs from that family.
Recent methods use diffusion on either graph embeddings or the discrete space of nodes and edges.
However, simple changes to embeddings (say, adding noise) can mean uninterpretable changes in the graph.
In discrete-space diffusion, each step may add or remove many nodes/edges.
It is hard to predict what graph patterns we will observe after many diffusion steps.
Our proposed method, called \ourmethod, takes a different approach.
We separate pattern generation and graph construction. 
To find patterns in the training graphs, we see how they transform vectors during random walks.
We then generate new graphs in two steps.
First, we generate realistic random walk ``trajectories'' which match the learned patterns.
Then, we find the optimal graph that fits these trajectories.
The optimization infers all edges jointly, which improves robustness to errors.
On four simulated and five real-world benchmark datasets,  \ourmethod outperforms existing methods.
The most significant differences are on large-scale graph structures such as PageRank, cuts, communities, degree distributions, and flows.
\ourmethod is also $10$x faster than its closest competitor.
Finally, \ourmethod is simple, needing only a transformer and standard optimizers. \\ \\ 
Code is available at \url{https://github.com/rahulnanda1999/GraphWeave}.

\end{abstract}

\section{Introduction}
\label{sec:intro}

Suppose that we have a set of molecules with some desirable property. For example, these molecules bind to a target protein to treat a disease. Our goal is to find other molecules that have this property. We can formalize this as a graph generation problem. Each molecule is a graph of atoms connected by bonds. The desirable property corresponds to some unknown patterns common to the given graphs. We want to generate new graphs that possess these patterns automatically. As another example, suppose we want to detect bots in a social network. Bots and regular users have different linkage patterns. However, we may have too few examples of bots to train a classifier.
To augment the training data, we can generate synthetic graphs whose link patterns match those of the known bots.
This can improve classification accuracy without increasing the cost.

In this paper, we tackle the problem of generating new graphs whose structure matches a set of training graphs. We do not consider node or edge features, which we can infer by a post-processing step. For instance, for molecule graphs, we can infer a node's feature (what atom it is) from its degree (number of bonds), and this determines its edge feature (bond strengths).
Now, even generating the graph structure is a complex problem.
Small-scale patterns (e.g., a benzene ring) might be important for some cases. In other applications, large-scale patterns may matter more (e.g., the ratios of various atoms, i.e., the degree distribution). No method can match all possible patterns.
Clarity about the patterns a method tries to match improves its interpretability.

However, existing methods rarely make their choices explicit.
One class of methods creates graphs by diffusion on the space of graph embeddings. They start from a random embedding, iteratively change it, and map the final embedding to a graph~\cite{evdaimon2024neuralgraphgeneratorfeatureconditioned,zhou2024unifyinggenerationpredictiongraphs}.
But even simple changes in embedding space (e.g., adding noise) may mean complex and unintuitive modifications to the graph structure.
Hence, the generative process is hard to interpret.

\begin{figure}[t]
    \centering
    \includegraphics[width=\textwidth]{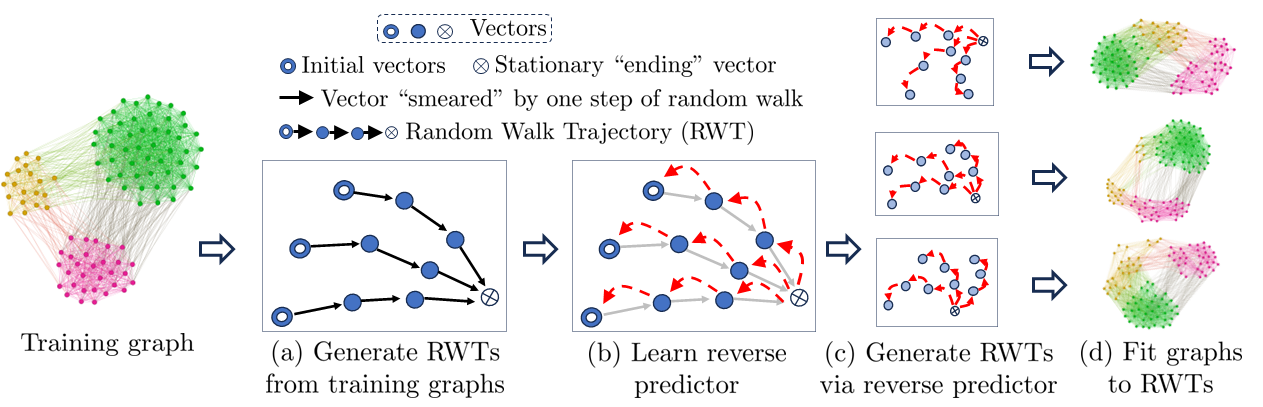}
    \caption{{\em Overview of \ourmethod:}
    (a)~Given one or more training graphs, we generate random walk trajectories (RWTs) from various starting vectors.
(b)~We learn to predict the previous step of a trajectory.
(c)~To generate a new graph, we first apply the reverse predictor several times on an ``ending'' vector (shown by $\otimes$).
We show that the ending vectors are simple to generate (Theorem~\ref{thm:endvec}).
(d)~Second, we find the optimal graph that fits the generated RWTs.
{\em The process works even with one training graph.}
Indeed, we generated the three graphs on the right from the single graph on the left (colors are added for clarity).
    }
    \vspace{-1em}
    \label{fig:overview}
\end{figure}

Another class of methods changes the graph structure instead of its embedding. For instance, some methods apply local changes to the graph in each iteration \cite{you2018graphrnn,shi2020graphaf}. Each step might add or remove a few nodes and edges, so the changes are intuitive.
However, these local changes must add up to the desired global patterns. The need for coordinated local changes makes the process sensitive to errors. Other approaches make global changes to the graph structure in each iteration \cite{vignac2022digress,bojchevski2018netgan}.
While this approach is very flexible, predicting what patterns will result from a series of complex changes is difficult.
This affects the interpretability of such methods.
\begin{center}
    \fbox{%
        \parbox{0.9\textwidth}{\centering\em How can we generate graphs matching multi-scale patterns in an interpretable way?}%
    }
\end{center}
\medskip
We make two design choices to achieve this goal.
First, we focus on {\em patterns that can be learned from random walks} on graphs. Specifically, we construct Random Walk Trajectories (RWTs) that track how vectors evolve over random walks.
We show that several standard graph families have unique RWT signatures. Hence, RWTs can intuitively capture helpful patterns.
Also, many applications are based on random walks.
So, graphs generated with the right RWT signatures can immediately positively impact such applications.

Our second design choice is to generate graphs via {\em optimization} on RWTs. In other words, we generate RWTs and find the optimal graph that fits these RWTs. Our approach separates the matching of patterns (via RWTs) from the graph construction (by optimization). 
This ``division of labor'' offers many benefits.
RWTs are easier to generate than graphs since RWTs are naturally in a vector space. Graph construction via optimization increases flexibility. For instance, we can impose constraints (e.g., sparsity) or add regularization for robustness.


\medskip\noindent
{\bf Our contributions:}
Our proposed method, named \ourmethod, generates Random Walk Trajectories (RWTs) and then optimally weaves them together into a coherent graph.
We discuss \ourmethod's advantages below.


\begin{enumerate}
\item \smallskip\noindent  {\bf Formulation:}
We cast graph generation as a two-step problem: generate realistic patterns and then optimize a graph to fit them.
The patterns we track are derived from random walk trajectories (RWTs).
The optimized graph is then helpful for any downstream tasks that rely on random walks.
\ourmethod's separation of pattern generation from graph optimization simplifies the generative process.
To our knowledge, {\em \ourmethod is the first to demonstrate this optimization-based approach.}

\item \smallskip\noindent  {\bf Interpretability:}
RWTs track how random walks affect vectors.
This basic process underlies many graph-theoretic problems.
For example, the vector could represent people's opinions in a social network. Then, the RWT would show how opinions evolve dynamically.
Hence, RWTs are easily interpretable.

\item \smallskip\noindent  {\bf Multi-scale structure:}
We show that RWTs can capture large-scale graph structures.
These include communities, flows, cut sizes, and degree distributions.
By varying the RWT initializations, we can also explore local structures, such as the neighborhoods of high-degree nodes.
Hence,  the set of RWTs of a graph can capture multi-scale structures.

\item \smallskip\noindent  {\bf Robustness:}
\ourmethod jointly optimizes all edges of the generated graph.
The optimization's inputs come from multiple RWTs.
Hence, the resulting graph is robust to occasional errors in the RWT generation process.

\item \smallskip\noindent  {\bf Simplicity:}
\ourmethod needs only a transformer to generate RWTs and a standard optimizer to find the best-fit graph.
Both of these are off-the-shelf tools.
Hence, \ourmethod's implementation is simple and reliable.

\item \smallskip\noindent {\bf Strong experimental results:}
On four simulated and five real-world datasets, \ourmethod outperforms state-of-the-art methods.
\ourmethod is particularly strong in matching large-scale graph structures like PageRank, cuts, communities, degree distributions, and flows.
Furthermore, \ourmethod is $10x$ faster than its closest competitor.
\end{enumerate}
\section{Proposed Method}
\label{sec:prop}

We are given a set $\mathcal{G}$ of undirected graphs, possibly of different sizes.
We want to generate new graphs that ``look like'' the graphs in $\mathcal{G}$.
For example, if $\mathcal{G}$ contains stochastic block model graphs, the generated graphs should match that family.

To generate such graphs, we must identify patterns from the graphs in $\mathcal{G}$.
Now, the space of all possible patterns is too large.
So, we must choose a subset of intuitive and widely applicable patterns.
We focus on random walk patterns since random walks underpin many graph applications.
Specifically, \ourmethod constructs random walk trajectories, as defined below.

\begin{definition}[Smoothed Random Walk Trajectory (RWT)]
\label{def:rwt}
An RWT has four parameters: (a) an adjacency matrix  $A\in\{0,1\}^{n\times n}$  of an undirected graph on $n$ nodes, (b) a function $f:\mathbb{R}_+\to\mathbb{R}_+$, (c) a smoothing parameter $\alpha\in(0,1)$, and (d) the number of steps $k$.
Let $d_i$ denote the degree of node $i$, and $d'_i:=(1-\alpha)d_i + \alpha$ the node's {\em smoothed} degree.
We assume that all nodes have positive degree.
Also, define the smoothed normalized adjacency matrix $L\in\mathbb{R}^{n\times n}$ and the ``starting vector'' $\bv\in\mathbb{R}^n$ as follows:
\begin{equation}
    \begin{aligned}
 L_{ij} &:=\frac{(1-\alpha)A_{ij} + \alpha\cdot\mathds{1}_{i=j}}{\sqrt{d'_i\cdot d'_j}} 
&  \bv_i &:=n\frac{f(d_i)}{\sum_j f(d_j)}.
\end{aligned}
\end{equation}
Then, the $k$-step Smoothed Random Walk Trajectory $RWT(A, f, \alpha, k)$ is the ordered sequence of vectors $\{\bv, L\bv, L^2\bv, \ldots, L^k\bv\}$.
\end{definition}



\begin{remark}
We use the normalized adjacency $L$ in Definition~\ref{def:rwt} instead of the random walk transition matrix $D^{-1}A$ since the symmetry of $L$ simplifies later steps.
We note that both matrices have the same eigenvalues and closely related eigenvectors.
\end{remark}

We can construct several RWTs for any graph by varying the function $f(\cdot)$.
For example, if $f(d_i)$ increases with $d_i$, the relative weight of high-degree nodes in the starting vector increases.
Then, the RWT explores the neighborhood of such nodes in more detail.

The smoothing parameter $\alpha$ in Definition~\ref{def:rwt} adds ``self-loops'' to all the nodes. The presence of self-loops slows down the random walk, leading to smoother trajectories. The higher the value of $\alpha$, the smoother the trajectory. We find that smoother trajectories are easier to predict and, hence, easier to generate. Next, we show several examples of RWTs.

\begin{example}[Erdos-Renyi Graphs]
\label{ex:random}
Suppose $\mathcal{G}$ contains Erdos-Renyi random graphs with connection probability $p$. In other words, the $j^{th}$ graph has $n^{(j)}$ nodes, and each node pair is linked with probability $p$. For simplicity, we ignore smoothing ($\alpha=0$). Then, all nodes in the $j^{th}$ graph have degree $\approx n^{(j)} p$ if $n(j)$ is large enough. Hence, for smooth $f(\cdot)$, every entry of this graph's starting vector is $\approx 1$. In other words, all Erdos-Renyi graphs, irrespective of their sizes, start their RWTs close to the all-ones vector $\bone$.
Furthermore, we can show that $L^{(j)}\bone \approx \bone$ for the normalized adjacency matrices $L^{(j)}$ of such graphs.
So, the RWTs of random graphs start near $\bone$, fluctuate around that point, and eventually converge.
\end{example}

\begin{example}[Stochastic Blockmodel (SBM)]
\label{ex:sbm}
Suppose $\mathcal{G}$ contains graphs sampled from an SBM with the following parameters.
There are two communities with sizes in the ratio $\beta:1-\beta$.
The probability of an edge between two nodes is $p$ if they are from the same community and $q$ otherwise.
Suppose we choose $f(x)=1$ for all $x$. 
Then, the starting vectors $\bv^{(j)}$ equal $\bone$ for all graphs.
We can show that for large enough $k$, the random walk vector $(L^{(j)})^k\bv^{(j)}$ converges to a vector $\bw$ with clustered entries.
Specifically, let 
\begin{align*}
 \kappa_1 & :=\beta p + (1-\beta)q, 
&  \kappa_2 &:=p+q-\kappa_1,
&  \nu:=\frac{\beta\sqrt{\kappa_1} + (1-\beta)\sqrt{\kappa_2}}{\beta\kappa_1+(1-\beta)\kappa_2}.
\end{align*}
Then, $\bw_i=\sqrt{\kappa_1}/\nu$ if node $i$ belongs to the first community, and $\sqrt{\kappa_2}/\nu$ otherwise (via Theorem~\ref{thm:endvec} proved later).
Thus, the RWT evolves from $\bone$ to the vector $\bw$ with clustered entries, irrespective of the graph's size.
\end{example}

\begin{example}[Preferential Attachment]
\label{ex:pa}
Suppose $\mathcal{G}$ contains graphs created from the Barabasi-Albert model \cite{albert2002statistical}.
Then, for any graph, the distribution of node degrees follows a power-law with exponent $3$.
As in the SBM example, take $f(x)=1$ for all $x$, so the starting vectors equal $\bone$.
The ending vectors of the random walks are proportional to the square roots of the degrees (Theorem~\ref{thm:endvec} later).
These follow a power-law distribution with exponent $5$.
\end{example}

\begin{example}[Expected-degree Random Graphs]
\label{ex:chunglu}
Suppose $\mathcal{G}$ contains random graphs whose expected degrees match those of SBMs.
Then, the starting and ending vectors will be the same as in Example~\ref{ex:sbm}, but the intermediate vectors will be different.
\end{example}

The above examples show that graphs from different families have different RWT signatures.
In all cases, the RWTs started from the all-ones vector.
But, they traced trajectories with different ending vectors.
For other choices of $f(\cdot)$, RWTs can explore (say) high-degree nodes and their neighborhoods.
\ourmethod automatically exploits such patterns to generate new graphs from the same family.


\medskip\noindent
{\bf Main Idea:}
\ourmethod has three steps.
First, we construct RWTs from the graphs in $\mathcal{G}$.
From these, we learn to reverse RWTs. In other words, given a vector from an RWT, \ourmethod learns to predict the previous vector.
Second, we use this reverse predictor to generate new RWTs. To do this, we need an ``ending'' vector. All the generated RWTs share the same ending vector but trace different trajectories. We show how to construct a realistic ending vector, and why we can think of the generated RWTs as being from the same graph. Third, from the generated RWTs, \ourmethod infers the underlying graph. Crucially, we infer all edges of this graph jointly. We can generate multiple graphs by repeating the second and third steps from different ending vectors. Next, we provide details for each of the three parts of \ourmethod.

\subsection{Learning to Reverse RWTs}
\label{sec:prop:reverse}

Given a set $\mathcal{G}=\{A_i\}$ of graphs and a set $\mathcal{F}=\{f_\ell(\cdot)\}$ of functions, we construct the set of all RWTs $\mathcal{R}:=\{RWT(A, f, \alpha, k)\}_{A\in \mathcal{G}, f\in\mathcal{F}}$.
Recall that each RWT is a sequence of vectors $\bv_1, \bv_2, \ldots, \bv_k$, where a pair $(\bv_j, \bv_{j+1})$ represents one step of a random walk on some graph in $\mathcal{G}$.

Next, we learn to reverse the RWTs, that is, to predict $\bv_j$ given $\bv_{j+1}$ and $f(\cdot)$.
The predictor must know $f(\cdot)$ since two different RWTs may arrive at the same $\bv_{j+1}$ via different paths.
Each path is determined by its starting vector, which depends on $f(\cdot)$.
To build the predictive model, we face two challenges.
\begin{itemize}
\item {\em Arbitrary length input/outputs:}
The length of the vectors $\bv_{j}$ and $\bv_{j+1}$ is the number of nodes in the graph.
Since the graphs in $\mathcal{G}$ can have different sizes, the lengths of vectors in $\mathcal{R}$ can vary.
\item {\em Permutation invariance:}
The model must be invariant to permutations of the components of $\bv_j$ and $\bv_{j+1}$ since a permutation is just a reordering of the nodes.
Such reordering should not affect the model's predictions.
\end{itemize}

Our solution is simple and elegant: we use a transformer.
Transformers can adapt to inputs of arbitrary context length.
In our case, the input vector is $\bv_{j+1}\in\mathbb{R}^n$, where the graph size $n$ varies between the graphs in $\mathcal{G}$.
For a transformer, this means a context of length $n$, where each item in the context is one-dimensional.
Given such an input, the transformer's output is also of length $n$, like the desired output vector $\bv_j$.
So, the same transformer can work for input/output vector pairs of all sizes. 
Also, a transformer without position embeddings or causal masking is invariant under permutations.
Thus, a vanilla transformer matches our desiderata.



However, we can significantly improve this transformer using embeddings.
Formally, we construct a binning function $B:\mathbb{R}_+\to[K]$ and an embedding $ \mathcal{E}:[K]\to\mathbb{R}^m$.
For vectors, we apply these functions elementwise.
In other words, $B(\bv)$ is the vector formed by applying $B(\cdot)$ to each element of $\bv$; $\mathcal{E}(\bv)$ is defined similarly.
The user chooses the number of bins $K$ and the embedding dimension $m$.
Higher values for $K$ and $m$ lead to more flexibility in the transformer.

\begin{figure*}[tbp]
    \centering
    \includegraphics[width=\textwidth]{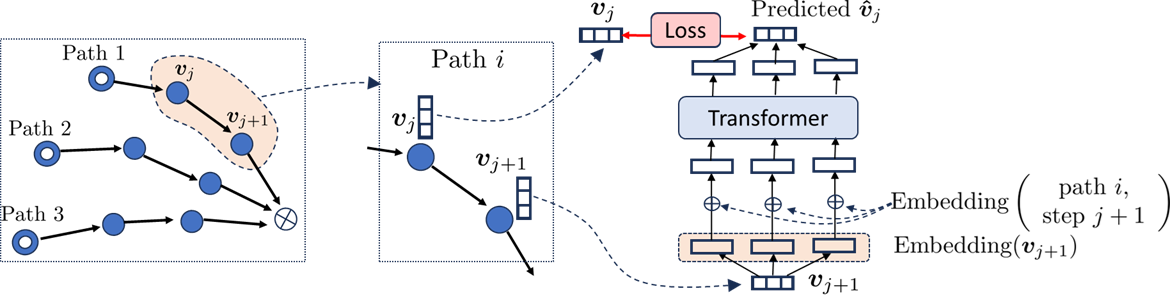}
    \caption{{\em Training the reverse predictor:} For every pair of successive vectors $(\bv_j, \bv_{j+1})$ from a path~$i$ of an RWT, we compare $\bv_j$ against a predicted vector $\hat{\bv}_j$ obtained from $\bv_{j+1}$. To create $\hat{\bv}_j$, the elements of $\bv_{j+1}$ are first converted into a sequence of embeddings. Then, we add embeddings reflecting the function $f$ used for the starting vector of path~$i$, and the step~$j+1$ within path~$i$. Finally, we transform these embeddings and project them to generate $\hat{\bv}_j$.}
    \label{fig:reverse}
\end{figure*}

Now, we preprocess the data to use these embeddings.
Specifically, in the RWTs, we replace each vector $\bv$ with $\bv \otimes \mathcal{E}(B(\bv))$.
We further augment each input vector with an embedding that encodes the choice of $f(\cdot)$ and the current step in the RWT.
Specifically, we define an embedding $\mathcal{E}'$ such that $\mathcal{E}'(f, j)\in\mathbb{R}^m$, for all functions $f\in\mathcal{F}$ and steps $j\in[k]$.
Given an input vector $\bv_{j+1}$ at step $j+1$ for function $f(\cdot)$, we define
\begin{align*}
T_{\Phi}(\bv_{j+1}, f, j+1) := \text{Transformer}_\Phi ( \bv_{j+1} &\otimes \mathcal{E}(B(\bv_{j+1}))\\
 + \bone &\otimes \mathcal{E}'(f, j+1)),
\end{align*}
where $\Phi$ represents the transformer's parameters.
The transformer's input is now a length-$n$ sequence of $m$-dimensional embeddings, as is the output.
Finally, we add a linear projection layer to convert the output back to $\mathbb{R}^n$:
\begin{align*}
P_{\Phi,\Psi}(\bv_{j+1}, f, j+1) &:= \text{Project}_\Psi\left( T_\Phi(\bv_{j+1}, f, j+1) \right) \text{ to } \mathbb{R}^n,
\end{align*}
where $\Psi$ are the projection parameters.
Given an input $\bv_{j+1}$, this projected output is our prediction for $\bv_j$.
We train the transformer to minimize the mean-squared error between the predicted and actual values of $\bv_j$.
\begin{align*}
\Phi, \Psi, \mathcal{E}, \mathcal{E}' &=
 \argmin \sum_{(\bv_j, \bv_{j+1})}\left(P_{\Phi, \Psi}(\bv_{j+1}, f, j+1) - \bv_j\right)^2.
\end{align*}
Figure~\ref{fig:reverse} illustrates this process.
After training, we have a reverse predictor $P_{\Phi,\Psi}$ (henceforth, $P$) such that
\begin{align}
P(\bv_{j+1}, f, j+1) \approx \bv_j \text{ for all } \bv_j\to\bv_{j+1} \text{ in the RWTs of $\mathcal{G}$.}
\label{eq:transformer}
\end{align}

\subsection{Generating RWTs}
\label{sec:prop:generateRWT}

Suppose we have learned an accurate reverse predictor.
Then, given an ``ending'' vector $\bar{\bv}_k$ and a choice of $f(\cdot)$, we can work backward by repeatedly predicting the previous vector: $\bar{\bv}_k\to\bar{\bv}_{k-1}\to\ldots\to\bar{\bv}_1$.
Different choices of $f(\cdot)$ yield different sequences for the same $\bar{\bv}_k$.
We call the sequence $\{\bar{\bv}_1, \ldots, \bar{\bv}_k\}$ a {\em generated} RWT.  

A generated RWT will only be realistic for special choices of $\bar{\bv}_k$.
In particular, we want $\bar{\bv}_k$ to be the ending vector for some graph from the same family as the graphs in $\mathcal{G}$.
Formally, we want $\bar{\bv}_k=\bar{L}^k \bar{\bv}_1$, where $\bar{L}$ is the smoothed normalized adjacency of such a graph.
But we do not know $\bar{\bv}_1$ or $\bar{L}$.
The following theorem shows how to select $\bar{\bv}_k$.


\begin{theorem}
Consider an $n$-node graph with degrees $\{d_i\}$, smoothed degrees $\{d'_i\}$, starting vector $\bv\in\mathbb{R}_+^n$ built using a function $f(d_i)$, and smoothed normalized adjacency $L$, as in Definition~\ref{def:rwt}.
Let $\bw$ be a vector with entries
\begin{align*}
\bw_i &:=\gamma\cdot \sqrt{d'_i}\\
\text{where }\gamma &=\frac{n\left(\sum_j f(d_j)\sqrt{d'_j}\right)}{\left(\sum_j f(d_j)\right)\left(\sum_j d'_j\right)}.
\end{align*}
Then, we have:
\begin{align*}
    \lim_{k\to\infty}\|L^k \bv - \bw\| &= 0.
\end{align*}
\label{thm:endvec}
\end{theorem}
\begin{proof}
From Lemma~\ref{lem:eig} in the Appendix, the largest eigenvalue of $L$ is $1$ with eigenvector $\bm{u}_1$ having components $\bm{u}_{1;i}=\sqrt{d'_i}/(\sum_j d'_j)$.
The matrix $L$ is irreducible (since the graph is connected) and aperiodic (since $\alpha>0$ induces self-loops).
Hence, no other eigenvalue has absolute value $1$.
Next, we observe that $\bv^T{\bm u}_1>0$, since both $\bv$ and $\bm u_1$ are in the positive orthant.
Hence, $L^k\bv$ tends to the vector $(\bv^T{\bm u}_1) {\bm u}_1$, which is seen to be the vector $\bw$.
\end{proof}

Theorem~\ref{thm:endvec} shows that a realistic ending vector $\bar{\bv}_k$ must be close to $\bw$, and $\bw$ only depends on the degrees $\{d_i\}$.
So, to construct $\bar{\bv}_k$, we must generate a realistic degree distribution that matches the family $\mathcal{G}$.
One approach is to sample a graph $G\in\mathcal{G}$ and perturb its degree distribution.  
Another option is to fit a model to the degree distributions of the graphs in $\mathcal{G}$.
For example, we can fit a power law or a lognormal since they are widely observed in real-world data.
Then, we can sample a new degree distribution from this fitted model.
Either way, we get a realistic degree distribution $\{d_i\}$.

Now, we generate an RWT as follows.
Using $\{d_i\}$ and a choice of $f\in\mathcal{F}$, we construct the ending vector $\bar{\bv}_k$ (Theorem~\ref{thm:endvec}).
Next, we use the reverse predictor $P$ from Section~\ref{sec:prop:reverse} to predict $\bar{\bv}_{k-1}:=P(\bar{\bv}_k)$, $\bar{\bv}_{k-2}:=P(\bar{\bv}_{k-1})$, and so on.
Proceeding this way, we construct all the vectors  $\{\bar{\bv}_1, \ldots, \bar{\bv}_k\}$. 
This sequence of vectors is the {\em generated RWT}.

By repeating these steps with the same $\bar{\bv}_k$ but different $f\in\mathcal{F}$, we can generate several RWTs.
We must now construct the sparse graph corresponding to the generated RWTs.
The following section shows how.

\begin{remark}[Differences from traditional diffusion]
Like diffusion models, \ourmethod has forward and reverse processes.
However, in diffusion, the forward process adds random noise.
In \ourmethod, the forward process is deterministic.
It represents the evolution of a vector during random walks.
Now, a forward process is only useful if it converges to an easy-to-sample stationary point.
In diffusion models, this point is often the standard Gaussian distribution (``100\%-noise'').
Theorem~\ref{thm:endvec} shows that \ourmethod's forward process also converges.
But our stationary vector represents the convergence of random walks, not noise.
Since random walks are widely used in applications, learning their patterns via RWTs can be more beneficial than learning a noise process on graphs.
\end{remark}

\begin{remark}[Single ending vector in Figure~\ref{fig:overview}]
Theorem~\ref{thm:endvec} shows that the ending vector is the same for all starting vectors after normalization by the appropriate $\gamma$.
Figure~\ref{fig:overview} shows this visually.
However, Definition~\ref{def:rwt} uses unnormalized starting vectors so that we can discuss Examples~\ref{ex:sbm}-~\ref{ex:chunglu} using the same starting vectors.
\end{remark}

\subsection{Inferring the Graph}
\label{sec:prop:infer}

Given a set of generated RWTs, we want to find one graph that generates them.
Suppose $\bar{\bv}_j$ and $\bar{\bv}_{j+1}$ are successive vectors in one of the generated RWTs.
Then, the smoothed normalized adjacency matrix $L$ for this graph must satisfy $L \bar{\bv}_j = \bar{\bv}_{j+1}$.
This relation holds for every pair of successive vectors.
Formally, let $V_1$ be a matrix with rows $\{\bar{\bv}_j\}$ and $V_2$ a matrix with rows $ \{\bar{\bv}_{j+1}\}$.
Then,
\begin{align}
   V_1 L &= V_2. \label{eq:Vt}
\end{align}

From Definition~\ref{def:rwt}, the matrix $L$ is of the form  
\begin{align}
L &= (D')^{-1/2}((1-\alpha)A + \alpha I)(D')^{-1/2}, \label{eq:L}
\end{align}
where $A$ is the adjacency matrix of the desired graph, $D'$ is a diagonal matrix with entries $D'_{ii}=(1-\alpha)d_i + \alpha$ and $d_i$ is the degree of node $i$.
Note that we know the $\{d_i\}$ since we generated the degree distribution as the first step in creating RWTs (Section~\ref{sec:prop:generateRWT}).


Plugging Eq.~\ref{eq:L} into Eq.~\ref{eq:Vt}, we find $A$ by solving:
\begin{align}
\text{minimize}_{A} & \sum_{i,j\in[n]} |X_{ij}| \label{eq:optObjInt}\\
\text{where } & X = V_1 (D')^{-1/2}((1-\alpha)A + \alpha I)(D')^{-1/2} - V_2, \nonumber\\
 & A_{ij} \in \{0, 1\} \text{ for all } i,j\in[n], \nonumber\\
 & A=A^T,
 \text{Trace}(A)=0,
 A\bone=\bm{d},\nonumber
\end{align}
where $\bm{d}$ is the vector with entries $d_i$.
The constraints ensure that $A$ is an unweighted graph with degrees $\bm{d}$.
Eq.~\ref{eq:optObjInt} is an Integer Linear Program that can be solved by standard tools such as Gurobi.

\begin{remark}
An alternative to Eq.~\ref{eq:optObjInt} is to relax the requirement $A_{ij}\in\{0,1\}$ to $A_{ij}\in[0,1]$.
This results in a convex problem:
\begin{align}
\text{minimize}_{\tilde{A}} & \sum_{ij} |X_{ij}|\label{eq:optObj}\\
\text{where } & X = 
\|V_1 (D')^{-1/2}((1-\alpha)\tilde{A} + \alpha I)(D')^{-1/2} - V_2\|_F^2, \nonumber\\
 & 0 \leq \tilde{A} \leq 1, 
 \text{Trace}(A)=0,
 \tilde{A} = \tilde{A}^T, 
 \tilde{A}\bone = \bm{d}.
 \nonumber
\end{align}
This results in a {\em weighted} graph $\tilde{A}$.
We can construct $A$ by rounding the entries of $\tilde{A}$ as follows:
\begin{align}
 A_{ij} &:= \mathds{1}_{\tilde{A}_{ij} > a^\star + b^\star\log d_i}, \label{eq:objOpt2}\\
 \text{where } a^\star, b^\star &= \argmin_{a,b} \frac{1}{n} \sum_{i=1}^n \left| \frac{\sum_j \mathds{1}_{\tilde{A}_{ij} > a+b\log d_i}}{d_i}-1 \right|.\nonumber
\end{align}
The choice of $(a^\star, b^\star)$ minimizes the relative error between the node degrees of $A$ and the desired degrees $\{d_i\}$.
We can select $(a^\star, b^\star)$ by grid search over a chosen range.
While this approach offers no guarantees for the objective of Eq.~\ref{eq:optObjInt}, it often works well in practice.
\end{remark}

\subsection{Overall Algorithm}
\label{sec:prop:overall}
Algorithm~\ref{alg:ourmethod} shows the pseudocode for \ourmethod.
We first build the RWTs for the graphs in $\mathcal{G}$.
Then, we train a transformer to reverse each step of the observed RWTs.
To generate a graph, we first generate a realistic degree distribution.
The degree distribution gives us the ending vector $\bar{\bv}_k$ (Theorem~\ref{thm:endvec}).
Starting from $\bar{\bv}_k$, we generate RWTs in reverse order by repeatedly applying the transformer (Eq.~\ref{eq:transformer}).
Finally, we infer the graph corresponding to the generated RWTs by solving Equation~\ref{eq:optObjInt}.
We can generate multiple graphs by reusing the transformer with different degree distributions.

\begin{algorithm}[ht]
    \small
    \begin{algorithmic}[1]
        \Function{\ourmethod}{$\mathcal{G}, \mathcal{F}, \alpha, k$}
            \State $\mathcal{R}\gets \cup_{A\in\mathcal{G}}\cup_{f\in\mathcal{F}} RWT(A, f, \alpha, k)$
            \State $\mathcal{D} \gets \{(\bv_j, \bv_{j+1}); \bv_j\to\bv_{j+1} \text{ in some RWT in } \mathcal{R}\}$
            \State Define $B:\mathbb{R}\to[K]$ \Comment{Binning function with $K$ bins}
            \State Define $B(\bv):=[B(\bv_1), B(\bv_2), \ldots, B(\bv_n)]^T$ for any $\bv\in\mathbb{R}^n$
            \vspace{1em}

            \LeftComment{{\bf Learn to reverse RWTs}}
            \State Define $\mathcal{E}:[K]\to\mathbb{R}^m$\Comment{value embedding function}
            \State Define $\mathcal{E}':|\mathcal{F}|\times k\to\mathbb{R}^m$
            \Comment{setting embedding function}
            \State Define $T_\Phi \gets \text{Transformer}:\mathbb{R}^{n\times m}\to\mathbb{R}^{n\times m} \text{ for any $n$}$
            \State $P_{\Phi,\Psi}(\bv, j, f) \gets \text{Project}_\Psi\left(T_\Phi(\bv\otimes \mathcal{E}(B(\bv)) + \bone\otimes \mathcal{E}'(f, j))\right)$
            \State $\Phi, \Psi, \mathcal{E}, \mathcal{E}'\gets \argmin \sum_{(\bv_j, \bv_{j+1})\in\mathcal{D}}\left(P_{\Phi,\Psi}(\bv_{j+1}, f, j+1) - \bv_j\right)^2$
            \vspace{1em}
            

            \LeftComment{{\bf Generate degree distribution}}
            \State $G\gets\text{sample graph from $\mathcal{G}$}$
            \State $\{d_i\} \gets \text{Perturbed degree distribution of $G$}$
            \State $d'_i \gets (1-\alpha)d_i+\alpha$ for all nodes $i$
            \vspace{1em}

            \LeftComment{{\bf Generate RWTs}}
            \State $V_1, V_2\gets \phi$
            \ForAll{$f\in\mathcal{F}$}
                \State $\bar{\bv}_k \gets \gamma\sqrt{\bm{d'}}$ \Comment{$\bm{d'}$ has entries $d'_i$; $\gamma$ is from Theorem~\ref{thm:endvec}}
                \State $\bar{\bv}_{k-j}\gets P_{\Phi,\Psi}(\bar{\bv}_{k-j+1}, f, k-j+1) \text{ for } j=1,2,\ldots, k-1$
                \State $V_1 \gets V_1\cup \{\bar{\bv}_{j}; j=1, \ldots, k-1\}$
                \State $V_2 \gets V_2\cup \{\bar{\bv}_{j+1}; j=1, \ldots, k-1\}$
            \EndFor
            \vspace{1em}

            \LeftComment{{\bf Infer graph}}
            \State $A \gets $ solve Equation~\ref{eq:optObjInt} using $V_1$ and $V_2$
            \State \Return unweighted graph with adjacency matrix $A$
        \EndFunction
    \end{algorithmic}
    \caption{\ourmethod}
    \label{alg:ourmethod}
\end{algorithm}

\medskip\noindent
{\bf Implementation details:}
For the binning function, we use $B(\bv):=\lfloor c(\bv-\mu)/\sigma\rfloor$, where $\mu$ and $\sigma$ are the mean and standard deviation of all vector entries in the training set, $c$ is a parameter that controls the number of bins, and the binning function is applied elementwise to $\bv$.
In our experiments, we set $c=3$, $\alpha=0.9$, and $k=10$.
For the set of functions $\mathcal{F}$, we use power laws: $\mathcal{F}=\{f:\mathbb{R}_+\to\mathbb{R}_+; f(d)=d^\beta, \beta\in\{\pm 1, \pm 2\}\}$.
We also simplify the form of $\mathcal{E}'(f, j)$ by adding an embedding of $f(\cdot)$ and and embedding of $j$.


\medskip\noindent
{\bf Computational complexity:}
We first consider the cost of training.
We construct $|\mathcal{F}|\times|\mathcal{G}|$ RWTs.
For each RWT, the main cost is the $k$ sparse-matrix-vector multiplications $L\bv_j$.
This takes $O(kE)$ time, where $E$ is the maximum number of edges in any graph in $\mathcal{G}$.
Hence, creating RWTs takes $O(|\mathcal{F}||\mathcal{G}|kE)$ time.
To train the transformer, we have $k|\mathcal{F}||\mathcal{G}|$ input vectors from the RWTs.
For each vector, the attention mechanism considers $O(n^2)$ pairs, where $n$ is the maximum number of nodes.
Each pair has a cost proportional to the embedding dimension $m$.
We assume that the transformer's size is fixed (i.e., $O(1)$ layers and heads).
So, the cost of training the transformer is $O(|\mathcal{F}||\mathcal{G}|kn^2m)$, and this is also the overall cost of training.

To generate a graph, we create its RWTs via $k|\mathcal{F}|$ passes of the transformer.
Each pass takes $O(n^2m)$ time.
Since Integer Linear Programs (Eq.~\ref{eq:optObjInt}) can have variable costs, we instead analyze the convex optimization (Eq.~\ref{eq:optObj}).
To generate a graph of $n$ nodes requires $O(n^2)$ parameters.
The main cost is in computing the matrix-matrix product of $V_1$ (size $k|\mathcal{F}|\times n$ and $(D')^{-1/2}\tilde{A}(D')^{-1/2}$ (size $n\times n$) in the objective.
Assuming we run gradient descent for a fixed number of steps, the convex optimization takes $O(\text{MatMult}(k|\mathcal{F}|\times n, n\times n))$ time.
The threshold step (Eq.~\ref{eq:objOpt2}) costs $O(n^2)$ for grid search.
Hence, the total cost of generation is $O(k|\mathcal{F}|n^2m + \text{MatMult}(k|\mathcal{F}|\times n, n\times n)) = O(k|\mathcal{F}|n^2m)$.

Note that the dominant costs are training a transformer, matrix multiplication, and convex optimization.
There are fast off-the-shelf libraries for all three.

\section{Experiments}
\label{sec:exp}
We ran experiments to compare the quality of graph generated by \ourmethod against state of the art competing methods.

\medskip\noindent
{\bf Comparison metrics:}
We consider measures of node centrality (degree and Pagerank), local neighborhoods (clustering coefficient and ORBIT scores), quality of random partitions (cut-size, conductance, and modularity), connectivity between random node pairs (max-flow and resistance), and overall connectivity (if the graph is connected or not). Apart from overall connectivity, each measure results in a vector $\phi(G)$ for any graph $G$  (e.g., the vector of node degrees, or the modularities of $100$ random partitions).
We then compute the relative error
\begin{align}
    \text{error}_\phi(\mathcal{G}_{gen}\mid \mathcal{G}_{test})
    &:=
    \left|
    \frac{\sum_{G_i\in\mathcal{G}_{gen}, G_j\in\mathcal{G}_{test}}\text{distance}(\phi(G_i), \phi(G_j))}
    {\sum_{G_i,G_j\in\mathcal{G}_{test}}\text{distance}(\phi(G_i), \phi(G_j))}
    \times
    \frac{|\mathcal{G}_{test}|}{|\mathcal{G}_{gen}|}
    -1\right|,
    \label{eq:quality}
\end{align}
where $\mathcal{G}_{gen}$ is the set of generated graphs, $\mathcal{G}_{test}$ the set of test graphs from the same family as the training data, and the distance function is the Wasserstein metric between any two vectors $\phi(G_i)$ and $\phi(G_j)$.
If the generated graphs fit the test distribution, the error is close to $0$.

\medskip\noindent
{\bf Competing methods:}
We compare \ourmethod against several state of the art methods:
DiGress~\cite{vignac2022digress}, 
GSDM~\cite{luo2023fast},
GRASP~\cite{minello2024graph}
GDSS~\cite{pmlr-v162-jo22a},
and GraphRNN~\cite{you2018graphrnn}.
Apart from GraphRNN, which is an autoregressive model, all the others use diffusion.
These methods are recent, and have been shown to outperform older methods.
Hence, we compare \ourmethod against these methods.


\medskip\noindent
{\bf Simulated datasets:}
We consider four types of simulated graphs: 
(a) a {\em stochastic blockmodel} with $3$ communities containing $50\%$, $30\%$, and $20\%$ of the nodes, and a connection probability of $0.8$ for nodes in the same community and $0.3$ otherwise, 
(b) a {\em Watts-Strogatz model} with $4$ edges per node and a rewiring probability of $0.3$, 
(c) a Barabasi-Albert {\em preferential attachment model}, 
and (d) a {\em expected-degree random graph model}, whose degrees are the same as the Stochastic Blockmodel.

\begin{table*}[t]
    \centering
    \scriptsize
    \begin{NiceTabular}{c@{\hspace{1em}}|c|cccccccccc}
    \RowStyle[cell-space-limits=3pt,bold]{\rotate}
    & 
    & \multicolumn{1}{c}{Degree}
    & \multicolumn{1}{c}{Pagerank}
    & \multicolumn{1}{c}{\parbox{5em}{Connected Graphs?}}
    & \multicolumn{1}{c}{Cut sizes}
    & \multicolumn{1}{c}{Conductance}
    & \multicolumn{1}{c}{Modularity}
    & \multicolumn{1}{c}{\parbox{5em}{Clustering Coefficient}}
    & \multicolumn{1}{c}{ORBIT}
    & \multicolumn{1}{c}{Max Flow}
    & \multicolumn{1}{c}{Resistance}
    
    \\\hline

\Block[c]{6-1}{\rotatebox{90}{\parbox{10em}{\centering Stochastic\\ Blockmodel}}}
& DiGress & $0.10$ & $0.42$ & $\checkmark$ & $3.15$ & $0.14$ & $0.11$ & $\mathbf{1.16}$ & $2.90$ & $1.50$ & $1.74$  \\
& GSDM & $18.14$ & $11.77$ & $\times$ & $35.78$ & $8.67$ & $24.15$ & $31.13$ & $19.65$ & $19.29$ & $791.98$  \\
& GDSS & $2.66$ & $0.34$ & $\checkmark$ & $23.79$ & $0.53$ & $4.13$ & $22.14$ & $16.54$ & $11.79$ & $28.65$  \\
& GRASP & $2.09$ & $12.45$ & $\checkmark$ & $16.01$ & $2.61$ & $1.45$ & $10.69$ & $10.89$ & $12.47$ & $39.99$  \\
& GraphRNN & $37.15$ & $5.71$ & $\times$ & $60.62$ & $8.56$ & $31.88$ & $18.54$ & $31.05$ & $32.63$ & $2170.56$  \\
\cline{2-12}
& {\bf GraphWeave} & $\mathbf{0.02}$ & $\mathbf{0.02}$ & $\checkmark$ & $\mathbf{0.02}$ & $\mathbf{0.03}$ & $\mathbf{0.10}$ & $4.27$ & $\mathbf{1.35}$ & $\mathbf{0.01}$ & $\mathbf{0.02}$  \\
\hline

\Block[c]{6-1}{\rotatebox{90}{\parbox{10em}{\centering Watts\\ Strogatz}}}
& DiGress & $0.06$ & $0.47$ & $\checkmark$ & $6.17$ & $0.14$ & $\mathbf{0.18}$ & $\mathbf{2.07}$ & $3.99$ & $0.46$ & $1.94$  \\
& GSDM & $2.72$ & $2.84$ & $\checkmark$ & $2281.93$ & $3.70$ & $8.41$ & $5.41$ & $9537.24$ & $313.18$ & $18.94$  \\
& GDSS & $2.50$ & $1.33$ & $\checkmark$ & $1347.60$ & $3.33$ & $7.22$ & $2.69$ & $3681.39$ & $176.23$ & $18.26$  \\
& GRASP & $2.98$ & $2.07$ & $\times$ & $3294.84$ & $3.50$ & $9.57$ & $11.85$ & $13067.07$ & $393.88$ & $19.05$  \\
& GraphRNN & $0.71$ & $5.13$ & $\times$ & $138.97$ & $5.41$ & $7.16$ & $2.67$ & $9.78$ & $25.58$ & $53.64$  \\
\cline{2-12}
& {\bf GraphWeave} & $\mathbf{0.02}$ & $\mathbf{0.18}$ & $\checkmark$ & $\mathbf{0.22}$ & $\mathbf{0.11}$ & $0.22$ & $2.90$ & $\mathbf{3.65}$ & $\mathbf{0.02}$ & $\mathbf{1.77}$  \\
\hline

\Block[c]{6-1}{\rotatebox{90}{\parbox{10em}{\centering Preferential\\ Attachment}}}
& DiGress & $0.09$ & $\mathbf{0.01}$ & $\checkmark$ & $1.80$ & $\mathbf{0.08}$ & $0.06$ & $\mathbf{0.04}$ & $\mathbf{0.02}$ & $0.30$ & $\mathbf{0.27}$  \\
& GSDM & $0.88$ & $4.78$ & $\checkmark$ & $1637.97$ & $3.90$ & $6.60$ & $9.48$ & $549.15$ & $225.39$ & $35.19$  \\
& GDSS & $0.97$ & $4.29$ & $\checkmark$ & $891.67$ & $3.57$ & $5.27$ & $4.69$ & $192.20$ & $121.26$ & $33.01$  \\
& GRASP & $2.73$ & $2.32$ & $\times$ & $48.04$ & $0.30$ & $0.86$ & $6.37$ & $3.65$ & $5.13$ & $32.46$  \\
& GraphRNN & $6.91$ & $5.94$ & $\times$ & $180.42$ & $5.53$ & $12.39$ & $2.69$ & $5.50$ & $18.81$ & $184.25$  \\
\cline{2-12}
& {\bf GraphWeave} & $\mathbf{0.01}$ & $0.21$ & $\checkmark$ & $\mathbf{0.01}$ & $0.09$ & $\mathbf{0.01}$ & $0.85$ & $0.27$ & $\mathbf{0.00}$ & $1.50$  \\
\hline

\Block[c]{6-1}{\rotatebox{90}{\parbox{10em}{\centering Random\\ Graph\\ with degrees\\ like SBM}}}
& DiGress & $0.05$ & $\mathbf{0.01}$ & $\checkmark$ & $\mathbf{0.01}$ & $0.14$ & $\mathbf{0.02}$ & $\mathbf{0.01}$ & $\mathbf{0.00}$ & $0.03$ & $\mathbf{0.03}$  \\
& GSDM & $4.92$ & $1.94$ & $\checkmark$ & $18.59$ & $0.47$ & $3.50$ & $19.43$ & $15.14$ & $8.34$ & $25.31$  \\
& GDSS & $2.80$ & $0.38$ & $\checkmark$ & $21.13$ & $0.44$ & $3.52$ & $20.12$ & $17.86$ & $10.22$ & $25.20$  \\
& GRASP & $1.52$ & $10.67$ & $\checkmark$ & $13.39$ & $2.07$ & $1.15$ & $10.23$ & $9.86$ & $10.33$ & $26.09$  \\
& GraphRNN & $37.85$ & $4.35$ & $\times$ & $54.74$ & $8.17$ & $29.26$ & $19.29$ & $32.32$ & $29.14$ & $1980.06$  \\
\cline{2-12}
& {\bf GraphWeave} & $\mathbf{0.02}$ & $0.03$ & $\checkmark$ & $0.03$ & $\mathbf{0.05}$ & $0.34$ & $10.17$ & $4.54$ & $\mathbf{0.02}$ & $0.06$  \\
\hline

    \end{NiceTabular}
    \caption{{\em Comparison on simulated datasets:} The quality of the generated graphs is measured via Eq.~\ref{eq:quality} (lower is better).
    \ourmethod outperforms other methods, especially for the large-scale metrics like degree distributions, cut sizes, conductance, and max-flow.
    Also, \ourmethod and GDSS are the only methods that always generate connected graphs.    
    }
    \label{tbl:sim}
    \vspace{-1em}
\end{table*}

\medskip\noindent
{\bf Real-world datasets:}
We also tested our method on five real-world benchmark datasets.
These include
(a) {\em Cora}
(b) {\em Citeseer}, and  
(c) {\em Pubmed}, where the nodes represent
documents and edges represent citation relationships from which we extract 3-hop ego networks~\cite{sen2008collective}.
We also use
(d) {\em Proteins}, containing molecular graphs with $100$ to $500$ nodes in each graph~\cite{dobson2003distinguishing},
and
(e) {\em QM9}, comprising stable organic molecules with up to nine heavy atoms~\cite{wu2018moleculenet}.

\medskip\noindent
{\bf Experimental settings:}
For each dataset and each method, we train on $100$ graphs and then generate (at least) $40$ graphs.
We compute various comparison metrics for each of the generated graphs, and compare them against unseen test graphs from the same dataset using Eq.~\ref{eq:quality}.

\begin{table*}[t]
    \centering
    \scriptsize
    \begin{NiceTabular}{c@{\hspace{1em}}|c|cccccccccc}
    \RowStyle[cell-space-limits=3pt,bold]{\rotate}
    & 
    & \multicolumn{1}{c}{Degree}
    & \multicolumn{1}{c}{Pagerank}
    & \multicolumn{1}{c}{Cut sizes}
    & \multicolumn{1}{c}{Conductance}
    & \multicolumn{1}{c}{Modularity}
    & \multicolumn{1}{c}{\parbox{5em}{Clustering Coefficient}}
    & \multicolumn{1}{c}{ORBIT}
    & \multicolumn{1}{c}{Max Flow}
    & \multicolumn{1}{c}{Resistance}
    
    \\\hline

\Block[c]{6-1}{\rotatebox{90}{\parbox{10em}{\centering Cora}}}
& DiGress & $1.73$ & $0.10$ & $0.99$ & $0.12$ & $0.09$ & $1.06$ & $0.25$ & $1.63$ & $\mathbf{0.05}$  \\
& GSDM & $0.34$ & $0.13$ & $7.53$ & $0.20$ & $0.38$ & $1.40$ & $5.37$ & $41.11$ & $2.13$  \\
& GDSS & $\mathbf{0.01}$ & $0.11$ & $32.23$ & $0.35$ & $0.60$ & $1.70$ & $95.62$ & $117.72$ & $2.26$  \\
& GRASP & $0.03$ & $5.30$ & $0.71$ & $0.85$ & $1.27$ & $1.14$ & $0.47$ & $7.15$ & $0.77$  \\
& GraphRNN & $13.13$ & $\mathbf{0.02}$ & $0.20$ & $0.44$ & $\mathbf{0.07}$ & $1.08$ & $0.30$ & $3.26$ & $7.18$  \\
\cline{2-11}
& {\bf GraphWeave} & $\mathbf{0.01}$ & $\mathbf{0.02}$ & $\mathbf{0.18}$ & $\mathbf{0.11}$ & $\mathbf{0.07}$ & $\mathbf{0.67}$ & $\mathbf{0.16}$ & $\mathbf{0.18}$ & $0.69$  \\
\hline

\Block[c]{6-1}{\rotatebox{90}{\parbox{10em}{\centering Pubmed}}}
& DiGress & $10.25$ & $\mathbf{0.04}$ & $0.36$ & $0.07$ & $0.45$ & $0.43$ & $0.45$ & $1.06$ & $0.18$  \\
& GSDM & $0.37$ & $0.16$ & $4.04$ & $\mathbf{0.02}$ & $0.06$ & $1.02$ & $4.07$ & $15.79$ & $1.42$  \\
& GDSS & $\mathbf{0.00}$ & $0.28$ & $14.71$ & $0.05$ & $0.20$ & $1.65$ & $80.79$ & $51.05$ & $1.57$  \\
& GRASP & $0.02$ & $17.31$ & $0.21$ & $2.65$ & $3.45$ & $1.45$ & $0.98$ & $0.67$ & $0.21$  \\
& GraphRNN & $11.50$ & $0.15$ & $0.26$ & $0.13$ & $\mathbf{0.04}$ & $0.38$ & $0.20$ & $1.07$ & $6.83$  \\
\cline{2-11}
& {\bf GraphWeave} & $0.07$ & $0.07$ & $\mathbf{0.03}$ & $0.11$ & $0.07$ & $\mathbf{0.26}$ & $\mathbf{0.03}$ & $\mathbf{0.06}$ & $\mathbf{0.01}$  \\
\hline

\Block[c]{6-1}{\rotatebox{90}{\parbox{10em}{\centering Citeseer}}}
& DiGress & $8.87$ & $\mathbf{0.04}$ & $0.60$ & $\mathbf{0.02}$ & $0.08$ & $1.19$ & $0.24$ & $0.99$ & $0.64$  \\
& GSDM & $0.21$ & $0.13$ & $3.71$ & $\mathbf{0.02}$ & $0.10$ & $1.49$ & $8.30$ & $20.45$ & $1.41$  \\
& GDSS & $0.11$ & $0.13$ & $11.42$ & $0.09$ & $0.27$ & $2.02$ & $90.31$ & $54.23$ & $1.54$  \\
& GRASP & $\mathbf{0.08}$ & $0.79$ & $2.45$ & $0.23$ & $0.32$ & $1.87$ & $16.37$ & $13.13$ & $0.75$  \\
& GraphRNN & $14.09$ & $0.10$ & $0.24$ & $0.17$ & $\mathbf{0.00}$ & $1.32$ & $\mathbf{0.05}$ & $1.91$ & $3.95$  \\
\cline{2-11}
& {\bf GraphWeave} & $0.17$ & $0.07$ & $\mathbf{0.03}$ & $0.13$ & $0.08$ & $\mathbf{0.48}$ & $0.09$ & $\mathbf{0.23}$ & $\mathbf{0.23}$  \\
\hline

\Block[c]{6-1}{\rotatebox{90}{\parbox{10em}{\centering QM9}}}
& DiGress & $0.06$ & $0.01$ & $0.11$ & $\mathbf{0.27}$ & $\mathbf{0.23}$ & $0.35$ & $0.08$ & $\mathbf{0.14}$ & $0.04$  \\
& GSDM & $0.25$ & $0.03$ & $2.60$ & $0.98$ & $0.76$ & $1.76$ & $4.09$ & $3.19$ & $0.91$  \\
& GDSS & $0.09$ & $0.20$ & $0.85$ & $0.61$ & $0.48$ & $0.94$ & $1.22$ & $1.01$ & $0.30$  \\
& GRASP & $\mathbf{0.01}$ & $0.11$ & $0.06$ & $0.54$ & $0.39$ & $\mathbf{0.03}$ & $0.12$ & $0.19$ & $\mathbf{0.03}$  \\
& GraphRNN & $\mathbf{0.01}$ & $0.09$ & $0.22$ & $0.63$ & $0.38$ & $0.64$ & $0.39$ & $0.29$ & $\mathbf{0.03}$  \\
\cline{2-11}
& {\bf GraphWeave} & $0.03$ & $\mathbf{0.00}$ & $\mathbf{0.02}$ & $0.39$ & $0.31$ & $0.47$ & $\mathbf{0.06}$ & $0.33$ & $0.07$  \\
\hline

\Block[c]{6-1}{\rotatebox{90}{\parbox{10em}{\centering Proteins}}}
& DiGress & $5.94$ & $0.33$ & $1.03$ & $1.87$ & $3.41$ & $7.19$ & $4.64$ & $4.17$ & $0.61$  \\
& GSDM & $0.74$ & $\mathbf{0.00}$ & $10.20$ & $0.10$ & $0.75$ & $3.74$ & $2254.53$ & $35.68$ & $1.63$  \\
& GDSS & $0.78$ & $0.01$ & $36.58$ & $0.19$ & $1.11$ & $\mathbf{2.95}$ & $22210.10$ & $96.64$ & $1.69$  \\
& GRASP & $0.89$ & $31.57$ & $2.51$ & $14.43$ & $8.98$ & $5.40$ & $1010.18$ & $10.48$ & $\mathbf{0.60}$  \\
& GraphRNN & $2.30$ & $0.16$ & $0.24$ & $0.19$ & $\mathbf{0.06}$ & $4.88$ & $\mathbf{1.51}$ & $4.01$ & $12.04$  \\
\cline{2-11}
& {\bf GraphWeave} & $\mathbf{0.01}$ & $0.03$ & $\mathbf{0.00}$ & $\mathbf{0.04}$ & $\mathbf{0.06}$ & $5.62$ & $2.33$ & $\mathbf{0.66}$ & $3.25$  \\
\hline

    \end{NiceTabular}
    \caption{{\em Comparison on real-world datasets (lower is better).} 
    \ourmethod is best, or close to best, for 
    most measures and datasets.
    }
    \label{tbl:real}
    \vspace{-1em}
\end{table*}

\begin{table*}[htbp]
    \centering
    \footnotesize
    
    \begin{NiceTabular}{c||c|c|c||c|c|c}
    
    \multicolumn{1}{c||}{\bf Improvement of}
    & \multicolumn{3}{c||}{\em Stochastic Blockmodel}
    & \multicolumn{3}{c}{\em Preferential Attachment}\\

    {\bf {\em Integer} over}    
    & \multicolumn{1}{c|}{$\mathbf{50}$ \bf nodes}
    & \multicolumn{1}{c|}{$\mathbf{100}$ \bf nodes}
    & \multicolumn{1}{c|}{$\mathbf{200}$ \bf nodes}
    & \multicolumn{1}{c|}{$\mathbf{50}$ \bf nodes}
    & \multicolumn{1}{c|}{$\mathbf{100}$ \bf nodes}
    & \multicolumn{1}{c}{$\mathbf{200}$ \bf nodes}
    \\    \hline
    {\bf \em Convex}
    & $31\%$ & $57\%$ & $18\%$
    & $58\%$ & $47\%$ & $21\%$\\
    {\bf \em Random}
    & $80\%$ & $91\%$ & $82\%$
    & $79\%$ & $67\%$ & $55\%$
    \end{NiceTabular}
    \caption{{\em Fidelity: of RWT reconstruction:}
    The graph generated by the Integer Linear Program (Eq.~\ref{eq:optObjInt}) is significantly better than the alternatives.}
    \label{tbl:rwtfidelity}
    \vspace{-1em}
\end{table*}

\medskip\noindent
{\bf Quality of graph generation:}
Table~\ref{tbl:sim} compares all competing methods for the simulated datasets.
We find that {\em \ourmethod generally outperforms the competing methods in measures of large-scale graph structures.}
For example, \ourmethod excels are predicting node degrees and cut sizes.
\ourmethod is also the best or close to the best for other metrics such as modularity, max-flow, and resistance.
The closest competing method is DiGress, but DiGress sometimes generates disconnected graphs.
In contrast, \ourmethod always generates connected graphs. 
Also, \ourmethod is significantly faster than DiGress, as we show later.

\ourmethod does particularly well for the Stochastic Blockmodel family of graphs.
This is because such graphs show large-scale community structure, and random walks can pick up such structure.

Table~\ref{tbl:real} compares the quality of all competing methods on the real-world datasets.
The results mirror those for the simulated datasets.
For large-scale measures such as the distribution of cut sizes, \ourmethod is the best on all datasets.
Furthermore, it is either the best or close to the best for degree distributions, Pagerank centrality distributions, conductance, and modularity.

\medskip\noindent
{\bf Effect of optimization:}
We compared our two optimization approaches: the Integer Linear Program of Eq.~\ref{eq:optObjInt} ({\em Integer}), and the convex relaxation with rounding of Eqs.~\ref{eq:optObj} and~\ref{eq:objOpt2} ({\em Convex}).
We also considered a baseline ({\em Random}) that picks a random graph with the same degrees as {\em Integer}.
The comparison metric is the objective function of Equations~\ref{eq:optObjInt} and~\ref{eq:optObj}, which measures how closely the generated graph matches the desired RWTs.

Table~\ref{tbl:rwtfidelity} shows that {\em Integer} is between $20\%-55\%$ better than {\em Convex}, and both are significantly better than {\em Random}.
The difference between {\em Integer} and {\em Convex} is because the latter needs to threshold edges from $[0,1]$ to $\{0,1\}$.
This thresholding step (Eq.~\ref{eq:objOpt2}) can increase the error in RWT reconstruction.

\medskip
\noindent
{\bf Wall-clock time:}
Figure~\ref{fig:time} compares the wall-clock times for the various methods.
We see that \ourmethod has the fastest training time, and has reasonable generation time.
Furthermore, {\em \ourmethod is $10x$ faster than its closest competitor (DiGress).}

\begin{figure}[h]
    \centering
    \begin{subfigure}{0.3\textwidth}
        \includegraphics[width=\textwidth]{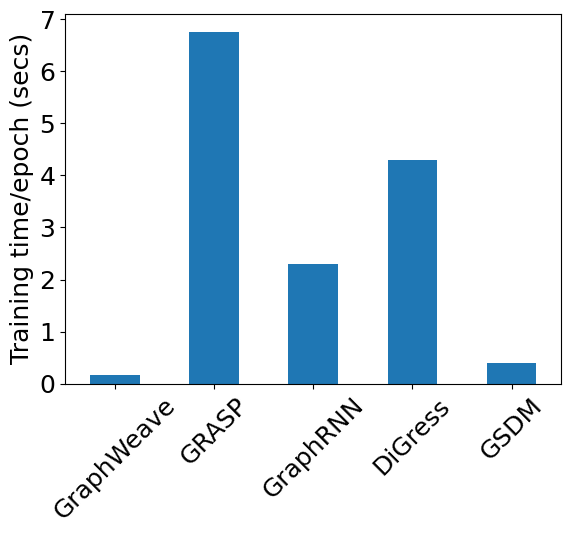}
        \caption{Training time}
    \end{subfigure}
    \begin{subfigure}{0.3\textwidth}
        \includegraphics[width=\textwidth]{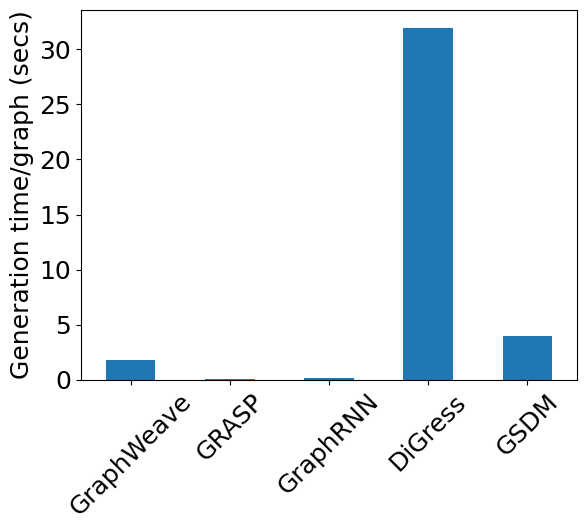}
        \caption{Generation time}
    \end{subfigure}%
    \caption{{\em Wall-clock times:} GDSS is much slower, and is not shown.}
    \label{fig:time}
    \vspace{-1em}
\end{figure}

\medskip \noindent
{\bf Sensitivity Analysis:}
We investigate the sensitivity of our measures to variations in the hyperparameters $c$, $\alpha$, and $k$.
Recall that $c$ controls the number of bins in the binning function $B(\bv)$, $\alpha$ smoothes the RWTs, and $k$ is the length of an RWT.
All the previous experiments used the baseline setting of $(c=3, k=10, \alpha=0.9)$.
We ran experiments varying one hyperparameter at a time.
We report all metrics normalized relative to their values in the baseline setting.
Hence, a normalized value greater than~$1$ implies worse performance than the baseline, and lower than~$1$ implies better performance.


Figure~\ref{fig:sensitivity} summarizes these results. In each plot, the horizontal red line at $1$ indicates the baseline level. Deviations from this line indicate how strongly a given metric is affected by changing the corresponding parameter.
Overall, no hyperparameter setting dominates the baseline setting.
We also observe that:
\begin{itemize}
    \item {\bf Varying $\bm c$} mainly affects \textit{cut sizes} and \textit{resistance}. Higher the value of~$c$, better the performance for {\em cut sizes}.
    \item {\bf Varying $\bm \alpha$} has a more pronounced effect, particularly at $\alpha=0.99$, where several metrics (e.g., \textit{pagerank}, \textit{resistance}) exhibit large deviations from baseline.
    Thus, too much smoothing can negatively affect \ourmethod's performance.
    \item {\bf Varying $\bm k$} impacts \textit{cut sizes} and \textit{resistance} more than other metrics. This is similar to the effect of varying $c$.
\end{itemize}

\begin{figure}[ht]
    \centering
    \includegraphics[width=\linewidth]{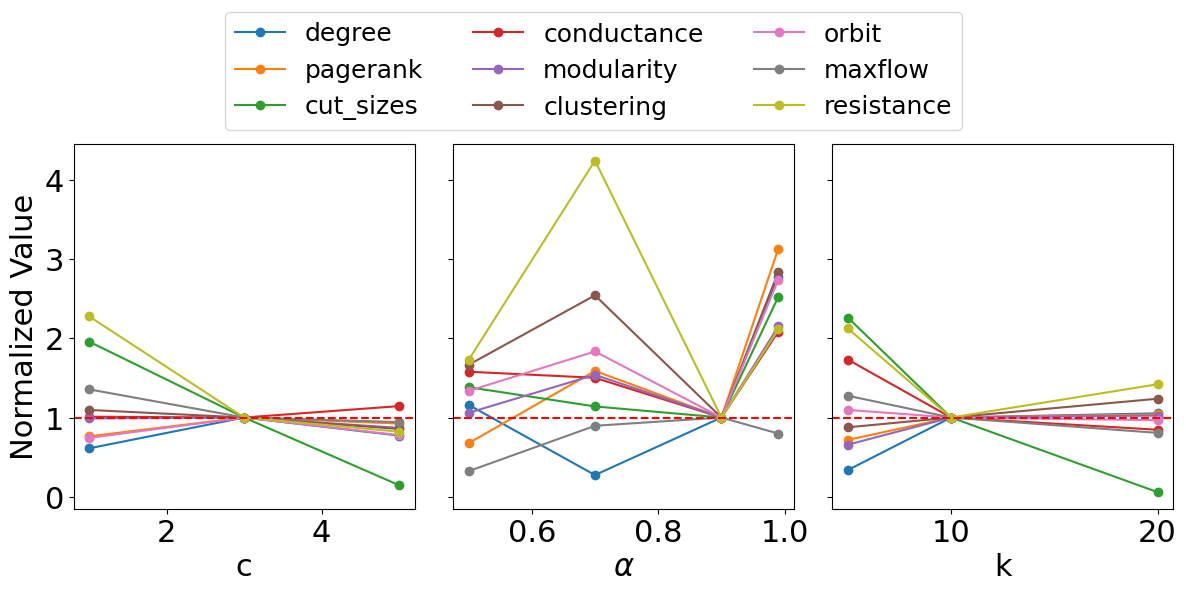}
    \caption{Sensitivity of normalized graph metrics to hyperparameter variations. Each metric is normalized to $1$ at the baseline $(c=3, k=10, \alpha=0.9)$, and deviations from $1$ indicate sensitivity to the corresponding parameter.}
    \label{fig:sensitivity}
\end{figure}

\section{Related Work}
\label{sec:related}
Graphs can be generated by autoregressive models, normalizing flow-based models, VAEs, GANs, and diffusion-based methods.
We discuss these below.

\smallskip\noindent
{\bf Autoregressive models:}
These generate graphs sequentially by adding one node or edge at a time.
Each step considers the previously generated structure.
The underlying method can be a recurrent neural network like GraphRNN~\cite{you2018graphrnn}), or attention mechanisms like GRAN~\cite{liao2019efficient}, or a combination with diffusion like GraphARM~\cite{pmlr-v202-kong23b}.
However, autoregressive models are often sensitive to node orderings, and node permutations can lead to divergent generation paths.


\smallskip\noindent
{\bf Normalizing flows:}
These methods provide a reversible transformation between graphs and a latent distribution, enabling easy likelihood computation.
GraphNVP~\cite{madhawa2019graphnvp} uses flows for molecule generation.
GraphAF~\cite{shi2020graphaf} introduces improvements to improve the quality and validity of the generated graphs.

\smallskip\noindent
{\bf VAEs and GANs:}
A VAE maps a graph to a latent space, and can reconstructs the graph from a latent embedding via probabilistic decoders~\cite{jin2018junction}.
GANs have also been applied to graph generation~\cite{bojchevski2018netgan}.
SPECTRE~\cite{martinkus2022spectre} integrates spectral features to enhance the GAN's expressivity.
MolGAN~\cite{de2018molgan} extends GANs for molecular graph generation by incorporating reinforcement learning.
However, many GAN-based models suffer from training instability and mode collapse, making them less reliable for diverse graph distributions.
Also, the black-box nature of adversarial training makes them hard to interpret~\cite{guo2022systematic}.



\smallskip\noindent
{\bf Diffusion models:}
Buoyed by the success of diffusion for image generation, these methods have come to the fore recently.
GDSS~\cite{pmlr-v162-jo22a} leverages a system of stochastic differential equations to jointly learn node and edge distributions.
DiGress~\cite{vignac2022digress} introduces a discrete diffusion model that edits node and edge attributes through Markov transitions.
GSDM~\cite{luo2023fast} applies diffusion on the spectrum of the adjacency matrix, while GRASP \cite{minello2024graph} focuses on the spectrum of the Laplacian.
While these methods are the state of the art, we show that \ourmethod outperforms them, particularly for large-scale structures like the distribution of cut sizes of random partitions and node Pagerank distributions.
Furthermore, \ourmethod is faster than its closest competitors.


\section{Conclusions}
\label{sec:conc}

To the question {\em ``How can we generate a graph with the right patterns?''}, we give a two-step answer: first generate the patterns, then optimize the graph.
We choose to focus on patterns that we can learn from random walks.
The reason is that many downstream applications use random walks, so graphs generated this way can have a significant impact.
The optimization step makes the graph robust to noise in the generated patterns.
It also lets us impose constraints on the graph, such as desired degree distributions.

\ourmethod puts this idea into practice via a fast, interpretable, and simple algorithm.
\ourmethod learns to predict random walk trajectories, which show how random walks transform a vector of node attributes. 
Then, using this predictor, we generate new trajectories.
Finally, we find the optimal graph that fits these trajectories.
The algorithm only requires a transformer and an optimizer. 
Experiments on several simulated and benchmark datasets show that \ourmethod outperforms the state of the art, and is among the fastest methods.

\appendix

\section{Smoothed Normalized Adjacency Matrix}
\begin{lemma}
    Let $A$, $d_i^\prime$, and $L$ be defined as in Definition~\ref{def:rwt}.
    Let $D' = \mathrm{diag}(d_i^\prime)$.
    Let $\beta_1 \geq \cdots \geq \beta_n$ be the eigenvalues of~$L$.
    Then, $\beta_1=1$ with the corresponding eigenvector being $(D'^{1/2} \mathbf{1})/\|D'^{1/2} \mathbf{1}\|$, and $\beta_n>-1$.
    \label{lem:eig}
\end{lemma}
\begin{proof}
We have $L=D'^{-1/2} \left( (1-\alpha) A + \alpha I \right) D'^{-1/2}$.
Since $A\mathbf{1}=D\mathbf{1}$, we have $L D'^{1/2} \mathbf{1}=D'^{-1/2}((1-\alpha)A+\alpha I)\mathbf{1}=D'^{-1/2}((1-\alpha)D+\alpha I)\mathbf{1}=D'^{1/2}\mathbf{1}$.
So, $1$ is an eigenvalue of $L$ with eigenvector $(D'^{1/2} \mathbf{1})/\|D'^{1/2} \mathbf{1}\|$.
To show that it is the largest eigenvalue, we show that $I-L$ is positive semidefinite.
We have
\begin{align*}
    I-L &=D'^{-1/2}(D' -(1-\alpha)A - \alpha I)D'^{-1/2}=(1-\alpha)D'^{-1/2}(D -A)D'^{-1/2}\\
    &=(1-\alpha)D'^{-1/2}D^{1/2}(I-D^{-1/2}AD^{-1/2})D^{1/2}D'^{-1/2}.
\end{align*}
Now, $D$ and $D'$ are positive definite, and so is $I-D^{-1/2}AD^{-1/2}$, since
\begin{align*}
   \bm{x}^T(I-D^{-1/2}AD^{-1/2})\bm{x} &=\sum_i \bm{x}_i^2 - \sum_{(i,j) \in E} \frac{2 x(i) x(j)}{ \sqrt{d_i' d_j'} } = \sum_{(i,j) \in E} \left( \frac{ x(i) }{ \sqrt{d_i'} } - \frac{ x(j) }{ \sqrt{d_j'} } \right)^2 \geq 0,
\end{align*}
for any $\bm{x}$.
So $I-L$ is positive semidefinite.
Similarly, we show that $L$'s smallest eigenvalue is greater than $-1$ by showing that $I+L$ is positive definite.
\begin{align*}
    I+L &= D'^{-1/2}(D' +(1-\alpha)A + \alpha I)D'^{-1/2}=(1-\alpha)D'^{-1/2}(D +A)D'^{-1/2} + 2\alpha D'^{-1}\\
    &= (1-\alpha)D'^{-1/2}D^{1/2}(I + D^{-1/2}AD^{-1/2})D^{1/2}D'^{-1/2} + 2\alpha D'^{-1}.
\end{align*}
The second term is positive definite. The first term is positive semidefinite since
\begin{align*}
    \bm{x}^T(I+D^{-1/2}AD^{-1/2})\bm{x} &=\sum_{(i,j) \in E} \left( \frac{ x(i) }{ \sqrt{d_i'} } + \frac{ x(j) }{ \sqrt{d_j'} } \right)^2 \geq 0.
\end{align*}
\end{proof}

\bibliographystyle{splncs04}
\end{document}